\PassOptionsToPackage{table}{xcolor}
\documentclass[11pt]{article}

\usepackage[final]{acl}
\usepackage{subcaption}
\usepackage{times}
\usepackage{latexsym}
\usepackage{amsmath}
\usepackage{arydshln}
\usepackage{multirow}
\usepackage{float}
\usepackage{graphicx}
\usepackage{svg}
\usepackage{adjustbox}
\usepackage{xcolor}
\usepackage{mathtools}
\usepackage{amssymb}
\usepackage[table]{xcolor}

\usepackage[T1]{fontenc}

\usepackage[utf8]{inputenc}

\usepackage{microtype}

\usepackage{inconsolata}

\usepackage{graphicx}
\usepackage{algorithm}
\usepackage{algpseudocode}
\usepackage{amsthm}
\theoremstyle{plain}
\newtheorem{theorem}{Theorem}
\newtheorem*{theorem*}{Theorem}

\newtheorem{lemma}{Lemma}

\newtheorem{assumption}{Assumption}

\title{Online Iterative Self-Alignment for Radiology Report Generation}

\author{
    Ting Xiao\textsuperscript{\rm 1},
    Lei Shi\textsuperscript{\rm 1},
    Yang Zhang\textsuperscript{\rm 2},
    HaoFeng Yang\textsuperscript{\rm 1},
    Zhe Wang\textsuperscript{\rm 1}\thanks{Corresponding author.},
    Chenjia Bai$^{3*}$ \\
    \textsuperscript{\rm 1}East China University of Science and Technology, 
    \textsuperscript{\rm 2}Tsinghua University, \\
    \textsuperscript{\rm 3}Institute of Artificial Intelligence (TeleAI), China Telecom \\
    \begin{tabular}{ccc}
        xiaoting@ecust.edu.cn & y80230058@mail.ecust.edu.cn & z-yang21@mails.tsinghua.edu.cn \\
        Y80240086@mail.ecust.edu.cn & wangzhe@ecust.edu.cn & baicj@chinatelecom.cn
    \end{tabular}
}

\begin{document}
\maketitle
\begin{abstract}
   Radiology Report Generation (RRG) is an important research topic for relieving radiologists' heavy workload. Existing RRG models mainly rely on supervised fine-tuning (SFT) based on different model architectures using data pairs of radiological images and corresponding radiologist-annotated reports. Recent research has shifted focus to post-training improvements, aligning RRG model outputs with human preferences using reinforcement learning (RL). However, the limited data coverage of high-quality annotated data poses risks of overfitting and generalization. This paper proposes a novel Online Iterative Self-Alignment (OISA) method for RRG that consists of four stages: self-generation of diverse data, self-evaluation for multi-objective preference data, self-alignment for multi-objective optimization and self-iteration for further improvement. Our approach allows for generating varied reports tailored to specific clinical objectives, enhancing the overall performance of the RRG model iteratively. Unlike existing methods, our framework significantly increases data quality and optimizes performance through iterative multi-objective optimization. Experimental results demonstrate that our method surpasses previous approaches, achieving state-of-the-art performance across multiple evaluation metrics.
\end{abstract}

\section{Introduction}\label{section_1}

Radiology Report Generation (RRG) aims to automatically generate free-text descriptions of radiology images by summarizing visual content and clinical insights. Due to its great potential for alleviating radiologists’ workload, many works have been proposed to perform supervised fine-tuning (SFT) for RRG on data pairs $(x, y)$ (where $x$ denotes the radiological image and $y$ represents the corresponding report annotated by radiologists) by refining different network architectures \cite{R2Gen, R2GenCMN, MAN} or incorporating external knowledge from knowledge graphs \cite{KIUT, jin2024promptmrg}, or fine-tuning from large vision/language models \cite{MedVersa, CheXagent}. However, due to the limited scale of the high-quality annotated data, such approaches raise concerns about overfitting and generalization beyond the dataset.

Recent research has shifted its focus to the post-training stage, improving the capabilities of existing RRG models by aligning their outputs with human preferences. For instance, CMN+RL \cite{qin2022reinforced}, \citet{delbrouck2022improving} and MPO \cite{MPO} utilize signals from natural language generation (NLG) metrics, entities and relationships, or clinical efficacy (CE) metrics as reward functions, applying reinforcement learning (RL) to align the RRG model. However, this RL alignment process may still be constrained by the data coverage of the training set \citep{xiong2024iterative}.
Alternatively, \citet{hein2024preference} construct a preference dataset using a strong 8B foundation model (i.e., CheXagent \cite{CheXagent}) and obtains preference labels via an LLM-based scoring model, GREEN \cite{GREEN}. Although such a method performs well, it requires a large foundation model and is limited to offline alignment on a fixed preference dataset.

Driven by this, we raise a question: ``\textit{\textbf{Is it possible for a lightweight RRG model to obtain a strong performance with the data generated by itself, breaking free from the limitation of the fixed dataset?}} '' To achieve this, we propose an Online Iterative Self-Alignment (OISA) method for RRG. OISA contains four steps: \emph{self-generation} to get diverse and unlimited data, \emph{self-evaluation} to obtain multi-objective preferences data, \emph{self-alignment} for multi-objective optimization, and \emph{self-iteration} of the above three steps to improve the performance of the RRG model further. Specifically,
1) Self-generation adopts a one-hot weight vector as a condition for the RRG model to represent radiologists' inherently heterogeneous and multi-objective inclination (e.g., report fluency, clinical accuracy). By changing the weight condition, the RRG model can generate diverse reports dedicated to specific objectives. 
2) Self-evaluation proposes a new preference data construction process, which includes data deduplication, evaluation, and stratified sampling to automatically construct a multi-objective preference dataset.
3) Self-alignment applies the Multi-Objective Direct Preference Optimization (MODPO) \cite{MODPO} algorithm based on the collected multi-objective preference datasets to optimize the initial RRG model, thereby improving the performance for multiple alignment objectives. 
4) Self-iteration uses the updated RRG model to generate preference data with higher quality and diversity. The three processes continuously improve the multi-objective alignment performance for the RRG model through iterations. 

Unlike previous post-training methods \cite{hein2024preference, MPO} in RRG, our method greatly extends the data coverage of the offline dataset via an iterative process and improves the performance of the RRG model via multi-objective preference optimization. Our main contributions are summarized as follows: 
(\romannumeral1) We propose a multi-objective RRG policy to generate multi-objective preference datasets in multiple rounds, which addresses the data limitation problem in previous methods. 
(\romannumeral2) We propose a self-improving process with the automatically constructed multi-objective preference data to improve the policy iteratively. The quality of the RRG reports can be continuously improved with the updated preference data and the theoretically grounded policy.
(\romannumeral3) Extensive experiments demonstrate that our method outperforms previous methods learned with fixed data or a single objective, and achieves state-of-the-art performance in multiple mainstream evaluation metrics.

\section{Preliminary}

Given an SFT RRG model ${\pi}_\text{sft}$, when prompted with a radiological image $x$, the RRG model ${\pi}_\text{sft}$ generates a response $y$ (i.e., report) via ${\pi}_\text{sft}(y \vert x)$. Based on ${\pi}_\text{sft}$, it is essential to align the model with human/radiologist preference via an RL from human feedback (RLHF) \cite{christiano2017deep} process. In RLHF, Direct Preference Optimization (DPO) is an effective preference learning method. Typically, given an prompt $x$ and an output text response $y$, a language model policy ${\pi}_{\theta}$ produces a conditional distribution ${\pi}_{\theta} (y \vert x)$. An ordinary RLHF method fits a reward model $r(x,y)$ to a human preference dataset $\mathcal{D}$ and then uses RL to optimize the model policy ${\pi}_{\theta}$ to generate responses that assign high rewards without deviating too far from the original reference model policy $\pi_{\mathrm{ref}}$. The overall objective can be formulated as:
\vspace{-7pt}
\begin{equation}\label{eq1}
\max_{\pi_\theta}\mathbb{E}_{x\sim\rho,y\sim\pi_\theta(y\vert x)} \left[r(x,y)-\beta \log \frac{\pi_\theta(y \vert x)}{\pi_{\mathrm{ref}}(y \vert x)}\right],
\end{equation}
where $\rho$ is the distribution prompt $x$ sampled from. $\beta$ is a coefficient that controls the deviation from the reference policy ${\pi}_{\mathrm{ref}}$, i.e., the initial SFT model ${\pi}_{\mathrm{sft}}$. In practice, ${\pi}_{\theta}$ is also initialized to ${\pi}_{\mathrm{sft}}$.

DPO simplifies the above objective by collecting the human preference data $\mathcal{D}$ and deriving an implicit reward function that fits the preference data. Here, $\mathcal{D}=\left\{(x_{i},y_{i}^w,y_{i}^l)\right\}_{i=1}^K$ consists of $K$ preference pairs $y^w$ and $y^l$, which denote the chosen and rejected responses to the same prompt $x$. 
Following the Bradley-Terry model \cite{bradley1952rank}, the probability of obtaining each preference pair is given by: $p(y^w\succ y^l)=\sigma(r(x,y^w)-r(x,y^l))$, where the subscript $i$ is omitted for simplicity, and $\sigma$ is the sigmoid function. In DPO, the objective described in Eqn. (\ref{eq1}) can be learned by applying the following loss over the preference data $\mathcal{D}$ as,
\vspace{-7pt}
\begin{equation}\label{eq2}
\begin{aligned}
 & \mathcal{L}_{\text{DPO}}(\pi_{\theta};\pi_{\mathrm{ref}})=-\mathbb{E}_{(x,y^{w},y^{l})\sim\mathcal{D}} \\
 & \left[\log\sigma\left(\beta \log\frac{\pi_\theta(y^w|x)}{\pi_{\mathrm{ref}}(y^w|x)}-\beta \log\frac{\pi_\theta(y^l|x)}{\pi_{\mathrm{ref}}(y^l|x)}\right)\right].
\end{aligned}
\end{equation}

In this paper, compared with the traditional SFT model ${\pi}_\text{sft}$, we introduce an additional weight vector $\mathbf{w}$ that represents radiologists' inherently heterogeneous and multi-objective inclination and use $\mathbf{w}$ as the conditional input for multi-objective alignment, where $\mathbf{w}=[w_{1},\ldots,w_{N}]$, $\mathrm{s.t.}\sum_{k=1}^{N} w_{k}=1$, $w_k$ is the weight of the $k$-th objective, $N$ is the number of objectives. 
Initially, we set the reference policy of the multi-objective RRG model as ${\pi}_\text{ref}(y|x,\mathbf{w})={\pi}_\text{sft}$, which can generate a report $y$ condition on prompt and weight vector. Then we denote the learned model policy as $\pi_{\theta_\mathbf{w}}(y|x,\mathbf{w})$. For each objective $k$, we use $M_k \in [M_1, \dots, M_N]$ as the evaluation function to determine the corresponding preference label by comparing the generated report with the ground truth report.

\begin{figure*}[!h]
    \centering
    \includegraphics[width=0.9\textwidth]{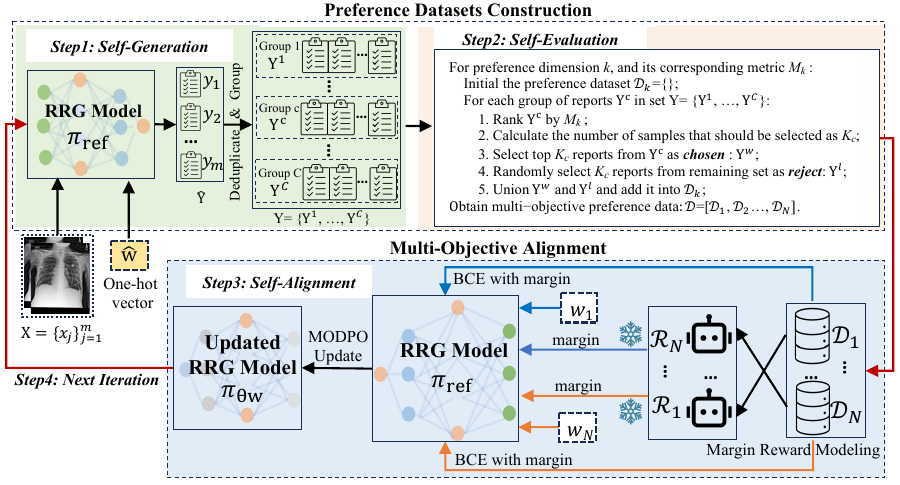}
    \caption{The illustration of the proposed OISA pipeline, comprising the Preference Dataset Construction (PDC) module and the Multi-Objective Alignment (MOA) module. The pipeline involves four steps: self-generation to obtain diverse data, self-evaluation to obtain a multi-objective preference dataset, self-alignment for multi-objective optimization, and self-iteration to further improve the performance of the RRG model.} \label{figure1}
\end{figure*}

\section{Methods}

Figure~\ref{figure1} illustrates the proposed OISA framework, comprising two modules and iterating through four steps to learn an RRG model policy ${\pi}_{{\theta}_{\mathbf{w}}}$. The first module is the Preference Data Construction (PDC), encompassing self-generation and self-evaluation processes. 
In self-generation, for an image prompt $x$, a one-hot weight vector $\mathbf{\hat{w}}$ is used as a condition for the initial RRG model ${\pi}_\text{ref}$ to generate reports tailored to a certain objective and deduplicate the generated reports to ensure data diversity. By adjusting the value of $\mathbf{\hat{w}}$, the RRG model can generate diverse reports tailored to different objectives.
In self-evaluation, a single evaluation metric $M_k \in [M_1, \dots, M_N ]$ serves as the scoring function for the $k$-th objective to construct its preference dataset $\mathcal{D}_{k}$ by stratified sampling. Then, performing a similar process for each objective can automatically construct a multi-objective preference dataset $\mathcal{D}= [ \mathcal{D}_{1},\ldots,\mathcal{D}_{N}]$. 

The second module is the Multi-Objective Alignment (MOA), i.e., a self-alignment process that aligns with multi-objective preferences. With the preference vector $\mathbf{w}$ and the preference data $\mathcal{D}$, a parametrized RRG model ${\pi}_{{\theta}_{\mathbf{w}}}$ is optimized via MODPO. Subsequently, the self-iteration process starts the next round of iteration by setting ${\pi}_\text{ref} \leftarrow {\pi}_{{\theta}_{\mathbf{w}}}$ to generate higher-quality preference data and conduct a new round of multi-objective alignment to further improve the performance of the model.

\subsection{Preference Datasets Construction}\label{section_3_2}

High-quality preference data is crucial for effective preference learning. However, obtaining rankings of reports across multiple dimensions from radiologists can be prohibitively expensive. Fortunately, the field of RRG offers various evaluation metrics, such as RadCliQ \cite{yu2023evaluating} and GREEN \cite{GREEN}, which show a strong correlation with radiologists’ evaluations. This enables us to use radiology-related metrics to represent radiologists' varying preferences and calculate evaluation scores relative to reference reports for ranking. 

In principle, a preference dataset $\mathcal{D}_k$ specific to evaluation metric $M_k$ should be constructed as follows:
1) Generate multiple responses for a prompt $x$ and rank the responses using the scoring metric $M_k$.
2) Select the highest-scoring response as $y^w$ and randomly select a response from the remaining ones as $y^l$. For a given promt set, a preference dataset $\mathcal{D}_k = \{(x_j, y_j^w, y_j^l)\}_{j=1}^{K}$ is obtained via above steps. However, since our lightweight SFT model cannot generate diverse responses for the same prompt, this paper proposes to construct a high-quality multi-dimensional preference dataset through self-generation and self-evaluation.

\textbf{Self-Generation.}
For each preference dimension $k$, we prompted ${\pi}_\text{ref}$ with a prompt set $\mathbf{X}=\{x_j\}_{j=1}^{m}$, and a one-hot weight vector $\mathbf{\hat{w}}_k=[w_{1},\ldots,w_{N}]$ (where $w_k=1$) that fully prefer the radiologist's inclinations for the $k$-th dimension, ${\pi}_\text{ref}$ will generate a report set $\hat{\mathbf{Y}} = \{(x_j, y_j)\}_{j=1}^{m}$ via ${\pi}_\text{ref}(\hat{\mathbf{Y}} \vert \mathbf{X}, \mathbf{\hat{w}}_k)$. 
To ensure data diversity, we deduplicate the report set $\hat{\mathbf{Y}}$ at the patient and disease levels. For reports from the same patient but different views, only the report with the highest BERTscore \cite{BERTScore} is retained. Subsequently, we group the remaining reports by disease label, as identified by CheXbert \cite{CheXbert}, and deduplicate reports within each group (more details about the deduplication are in \textcolor{blue}{Appendix \ref{appendix:pdc}.}), resulting in the candidate report set $\mathbf{Y} =\{\mathbf{Y}^c\}_{c=1}^{C}$, where $C$ is the total number of groups, $\mathbf{Y}^c=\{(x_j, y_j)\}_{j=1}^{N_c}$, and $N_c$ is the total number of samples in $c$-th group.

\textbf{Self-Evaluation.} For the $k$-th preference dimension, we apply evaluation metric $M_k$ as its scoring function, and the corresponding preference dataset $\mathcal{D}_k$ is constructed as follows:
(\romannumeral1) For each grouped report set $\mathbf{Y}^c \in \mathbf{Y}$, we evaluate it with evaluation metric $M_k$;
(\romannumeral2) Calculate the number of samples that should be selected from $\mathbf{Y}^c$, denoted as $K_c$;
(\romannumeral3) Select $K_c$ reports from $\mathbf{Y}^c$ with the highest $M_k$ scores as the \emph{chosen response} set $\mathbf{Y}^w=\{(x_j, y_j^w)\}_{j=1}^{K_c}$, whose corresponding prompt set is $\mathbf{X}^c=\{x_j\}_{j=1}^{K_c}$; 
(\romannumeral4) For each $x \in \mathbf{X}^c$, randomly select one report from the remaining report set ${\mathbf{Y}^c}_{- \mathbf{Y}^w}$ as its \emph{reject response}, resulting in the rejection response set $\mathbf{Y}^l=\{(x_j,y_j^l)\}_{j=1}^{K_c}$;
(\romannumeral5) Union $\mathbf{Y}^w$ and $\mathbf{Y}^l$ by $\mathbf{X}^c$, and add the results into $\mathcal{D}_k$; (\romannumeral6) Perform the above (\romannumeral1-\romannumeral5) steps to all groups will obtain the preference dataset $\mathcal{D}_k = \{(x_j, y_j^w, y_j^l)\}_{j=1}^{K}$. The whole process is summarized in Figure~\ref{figure1}, and the details of how to calculate $K_c$ are in \textcolor{blue}{Appendix \ref{appendix:KC}.}

Performing the above self-generation and self-evaluation processes for all preference dimensions can automatically construct multi-objective preference dataset $\mathcal{D} =[\mathcal{D}_1, \dots, \mathcal{D}_N]$. 

\subsection{Multi-Objective Alignment} \label{section_3_3}

To address the diversity of human preferences, we adopt MODPO \cite{MODPO}, which extends DPO at minimal cost to achieve multi-objective alignment. MODPO integrates a weighted combination of objectives and the training of RRG models into preference learning, consisting of two steps: marginal reward modeling and policy learning.

\textbf{Marginal reward modeling.} For each preference dataset $\mathcal{D}_k$, we parametrize the margin reward model $\mathcal{R}_k$ with $\pi_{{\theta}_\mathbf{w}}$, and train it through the $\mathcal{L}_{\text{DPO}}$ loss in Eqn. (\ref{eq2}). Thus, the marginal reward can be calculated as:
\vspace{-5pt}
\begin{equation}
    \label{reward-margin}
    \mathcal{R}_{k}(x, y) =  {\beta} \log \frac{\pi_{{\theta}_{\mathbf{w}}}(y \vert x,\mathbf{\hat{w}}_k)}{\pi_{\text{ref}}(y \vert x,\mathbf{\hat{w}}_k)}
\end{equation}

\textbf{RRG policy learning.} To align multi-objective preferences, we train the model on preference datasets sequentially. For each preference dimension $k$, we optimize the target policy $\pi_{\theta_{\mathbf{w}}}$ over the preference dataset $\mathcal{D}_k$ using the weight $\mathbf{w}$ with the $\mathcal{L}_{\text{MODPO}}$ loss as follows:

{\small
\begin{equation}
\begin{gathered}
    \mathcal{L}_{\mathrm{MODPO}}(\pi_{\theta_{\mathbf{w}}};\mathcal{R}_{-k},\pi_{\text{sft}},\mathcal{D}_{k})= \\
    -\underset{\mathcal{D}_{k}}{\operatorname*{\mathbb{E}}} \left[ \log \sigma \left( \frac{\beta}{{w}_{k}}  \log\frac{\pi_{\theta_{\mathbf{w}}}(y^{w}|{x},\mathbf{w})}{\pi_{\text{sft}}({y}^{w}|{x},\mathbf{w})} - \frac{\beta}{{w}_{k}} \log \frac{\pi_{\theta_{\mathbf{w}}}({y}^{l}|{x},\mathbf{w})}{\pi_{\text{sft}}({y}^{l}|{x},\mathbf{w})}\right.\right.\\
-\underbrace{ \frac{1}{{w}_{k}}\mathbf{w}_{-k}^{\top} \big( \mathcal{R}_{-k}({x},{y}^{w})-\mathcal{R}_{-k}({x},{y}^{l}) \big) }_{\text{margin,}m_{\phi}({x},{y}^{w},{y}^{l})} \bigg) \bigg],
\end{gathered}
\end{equation}}
where ${w}_k$ represents the $k$-th element of the preference vector $\mathbf{w}$, and $\mathbf{w}_{-k}$ represents all elements of $\mathbf{w}$ except for ${w}_k$; $\mathcal{R}_{-k}$ denotes the marginal reward model excluding $\mathcal{R}_{k}$. 
Then, policy model $\pi_{\theta_{\mathbf{w}}}$ continues to be trained on the preference dataset $\mathcal{D}_{k+1}$ until all preference datasets are traversed.

The loss function $\mathcal{L}_\text{MODPO}$ aligns the model with multi-objective preferences by incorporating preference vectors $\mathbf{w}$ as prompts during training, where each preference dataset $\mathcal{D}_k$ is associated with different weight $\mathbf{w}$ sampled from the weight space. Iterating over the whole weight space and optimize $\mathcal{L}_\text{MODPO}$ for each $\mathbf{w}$ to produce an empirical front ${\pi}_{\theta_{\mathbf{w}}}$ that aligned with the multi-objective preferences encoded in the datasets.

\subsection{Self-Iteration}\label{section_3_4}

After the above three steps, we get the updated model $\pi_{\theta_{\mathbf{w}}}$, and start the self-iteration process by  ${\pi}_\text{ref}\leftarrow{\pi}_{\theta_{\mathbf{w}}}$. The whole process is continuously iterated, which can greatly expand the data coverage of the offline dataset and improve model performance through multi-objective preference optimization. The whole pipeline is as follows:
\vspace{-5pt}
\begin{equation}
    \scalebox{0.6}{$
    \cdots \rightarrow 
    \underbrace{ \hspace{0cm} \{ \pi_\text{ref}^{(i)} \} \xrightarrow{\text{PDC}} 
    \overbrace{\{ \mathcal{D}_{k}^{(i)} \}_{k=1}^{N} \xrightarrow{\text{Eq(2)}} \{ \mathcal{R}_{k}^{(i)} \}_{k=1}^{N} \xrightarrow{\text{Eq(4)}} \{ {\pi}^{(i)}_{\theta_{\mathbf{w}}}\}}^{\text{Training}}}_{{\text{iteration} ~~i}} \rightarrow  \cdots
    $}
\end{equation}

\section{Theoretical Analysis}

We provide a brief theoretical analysis of our proposed method to demonstrate the effectiveness of online iterative self-alignment. Formally, we make the following assumption on the parameterization of the reward on each preference dimension $k$.
\begin{assumption}[Linear Reward]
\label{assump:linear_r}
The reward lies in the family of linear functions $ \scriptstyle r_\theta(x, y) = \langle \theta, \phi(x,y) \rangle = \theta^{\top} \phi(x, y)$ for some known and fixed $\scriptstyle \phi(x, y): \mathcal{X} \times \mathcal{Y} \rightarrow \mathbb{R}^d$ with $\scriptstyle \max_{x, y} \| \phi(x,y) \|_2 \leq 1$. Let $\theta^{\star}$ be the true parameter for the ground-truth reward function. To ensure the identifiability of $\theta^{\star}$, we let $\theta^{\star} \in \Theta_{B}$, where
\begin{align}
    \Theta_{B} = \{ \theta \in \mathbb{R}^d \big| \langle 1, \theta \rangle = 0, \| \theta \|_2 \leq B \}.
\end{align}
\end{assumption}
In our method, we obtain an estimated reward model $\mathcal{R}_k$ reparameterized with Eqn. (\ref{reward-margin}) via maximum likelihood estimation in Eqn. (\ref{eq2}). According to \citet{zhu2023principled}, we can bound the estimation error of $\mathcal{R}_k$ conditioned on the data $\mathcal{D}_k = \{ (x_j, y_j^w, y_j^l )\}_{j = 1}^{K}$.
\begin{lemma}\label{lemma:bound_MLE}
{\rm \citep{zhu2023principled}} For any $\lambda > 0$, letting $\gamma = 1/(2 + e^{-B} + e^{B})$, with probability at least $1 - \delta$, we have
{\small
\begin{align}
    \| \hat{\theta}_{\text{\rm MLE}} - \theta^{\star} \|_{\Sigma_{\mathcal{D}_k} + \lambda I} \leq \varrho := C \cdot \sqrt{\frac{d + \log (\frac{1}{\delta})}{\gamma^2 K} + \lambda B^2}, 
\end{align}}
where {\small $\Sigma_{\mathcal{D}_k} = \frac{1}{K} \sum\limits_{j = 1}^{K}( \phi(x_j, y_j^w) - \phi(x_j, y_j^l) )( \phi(x_j, y_j^w) - \phi(x_j, y_j^l) )^{\top}.$}
\end{lemma}
Given a preference weight vector $\mathbf{w} \in \mathbb{R}^{N}$ and a set of reward models over each preference dimension $\mathbf{r} = [\mathcal{R}_1, \ldots, \mathcal{R}_N]^\top$, the objective during RL tuning phase in multi-objective RLHF is,
\begin{align}
    J(\pi_\mathbf{w}, \mathbf{r}) = \mathbb{E}_{x, y} \left[ \mathbf{w}^\top \mathbf{r}(x, y) - \beta \log \frac{\pi_{\mathbf{w}}(y|x)}{\pi_{\text{ref}}(y|x)} \right].
\end{align}
where $x \sim \rho$ and $y \sim \pi_{\mathbf{w}}(y|x)$. For clarity, we define RL tuning objective $J(\pi_\mathbf{w}, \hat{\boldsymbol{\theta}}_{\text{MLE}})$ with MLE estimated reward $\hat{\boldsymbol{\theta}}_{\text{MLE}} = [\hat{\theta}_{1}, \ldots, \hat{\theta}_{N}]$ and $J(\pi_\mathbf{w}, \boldsymbol{\theta}^{\star})$ with ground-truth reward $\boldsymbol{\theta}^{\star} = [\theta^\star_1, \ldots, 
\theta^\star_N]$ respectively, as
{\small
\begin{align*}
    J(\pi_\mathbf{w}, \hat{\boldsymbol{\theta}}_\text{MLE}) &= \mathbb{E}_{x, y} \left[ \mathbf{w}^\top \hat{\boldsymbol{\theta}}_{\text{MLE}}^\top \phi(x, y) - \beta \log \frac{\pi_{\mathbf{w}}(y|x)}{\pi_{\text{ref}}(y|x)} \right], \\
    J(\pi_\mathbf{w}, \boldsymbol{\theta}^\star) &= \mathbb{E}_{x, y} \left[ \mathbf{w}^\top {\boldsymbol{\theta}^{\star}}^\top \phi(x, y) - \beta \log \frac{\pi_{\mathbf{w}}(y|x)}{\pi_{\text{ref}}(y|x)} \right].
\end{align*}}

Now we are ready to analyze the sub-optimality gap of the optimal policy $\hat{\pi}_\mathbf{w}$ derived from optimizing $J(\pi_\mathbf{w}, \hat{\boldsymbol{\theta}}_\text{MLE})$. For the output policy $\hat{\pi}_\mathbf{w} = \mathop{\arg\max}_\pi J(\pi_\mathbf{w}, \hat{\boldsymbol{\theta}}_\text{MLE})$, we have the following theoretical guarantee.
\begin{theorem}\label{thm:subop_gap}
For any $\lambda > 0$, $\beta > 0$, with probability at least $1 - \delta$, the optimal policy $\hat{\pi}_\mathbf{w}$ w.r.t. the objective $J(\pi_\mathbf{w}, \hat{\boldsymbol{\theta}}_\text{MLE})$ satisfies,
{\small
\begin{align}
    \text{\rm SubOpt} (\hat{\pi}_\mathbf{w}) \leq 2\varrho \cdot \sum_{k=1}^N w_k \| \mathbb{E}_{x \sim \rho} [ \phi(x, \pi^\star(x)) ] \|_{(\Sigma_{\mathcal{D}_k} + \lambda I)^{-1}},
\end{align}}
where $\pi^\star = \mathop{\arg\max}_{\pi} J(\pi_\mathbf{w}, \boldsymbol{\theta}^\star)$ is the optimal policy w.r.t. the ground-truth reward model $\boldsymbol{\theta}^\star$.
\end{theorem}
The proof is deferred to \textcolor{blue}{Appendix~\ref{appendix:proof_of_subop_gap}}. Here the term $\| \mathbb{E}_{x \sim \rho} [ \phi(x, \pi^\star(x)) ] \|_{(\Sigma_{\mathcal{D}_k} + \lambda I)^{-1}}$ can be interpreted as a measure of how well the current dataset $\mathcal{D}_k$ covers the distribution of responses generated by the target policy $\pi^\star$.
Theorem~\ref{thm:subop_gap} implies that when MODPO in each iteration guides the output policy $\hat{\pi}_\mathbf{w}$ towards the optimal policy $\pi^\star$, adopting an online iterative MODPO paradigm can tighten the bound by collecting a new preference dataset with higher quality generated by the new policy, thereby enjoying a theoretically grounded improvement. Besides, we emphasize that although we start by introducing a relatively simple linear reward assumption, the theoretical results are also ready to be extended to general function approximation \citep{chen2022human, wang2023rlhf}.

\section{Experiments}

\subsection{Datasets and Experiment Setting}\label{section_4_1}

\textbf{Datasets.} 
We evaluate our method on two public datasets, MIMIC-CXR \cite{johnson2019mimic} and IU-Xray \cite{demner2016preparing}. MIMIC-CXR is the largest public dataset for RRG, containing 337,110 chest X-ray images and 227,835 corresponding reports. Following the official split, MIMIC-CXR is randomly split into 7:1:2 for train, val, and test. We use the training set of MIMIC-CXR to construct preference datasets and test our model on its test set. IU-Xray consists of 7,470 chest X-ray images, accompanied by 3,955 reports. For IU-Xray, we follow PromptMRG \cite{jin2024promptmrg} to test our model on the entire IU-XRay set.

\textbf{Evaluation Metrics.} 
We evaluate our model by comparing the generated report with the corresponding ground truth report using seven metrics from two domains: NLG and radiology. For NLG metrics, we include BLEU1, BLEU4 \cite{DBLP:conf/acl/PapineniRWZ02} and BERTScore \cite{BERTScore}. For radiology metrics, we consider GREEN \cite{GREEN}, RadGraphF1 \cite{jain2021radgraph}, CheXbertF1 \cite{CheXbert} and RadCliQ \cite{yu2023evaluating}. For all metrics, except RadCliQ, a higher value signifies better performance. A detailed explanation of each metric is in \textcolor{blue}{Appendix~\ref{appendix:metrics}}.

\textbf{Implementation details.} 
We set PromptMRG \cite{jin2024promptmrg} as our baseline model and conduct three rounds of iterations, each consisting of 60 epochs. In each round, the preference dataset is constructed based on the model trained in the previous round, conditioned on a one-hot weight vector. We use three radiology metrics, RadCliQ, RadGraphF1, and GREEN to represent different objectives, resulting in a preference dataset with $N=3$ objectives, $\mathcal{D} =[\mathcal{D}_\text{RadCliQ}, \mathcal{D}_\text{RadGraphF1}, \mathcal{D}_\text{GREEN}]$ with $K = 10,000$ pair of data in each dataset. During training, we apply different $\mathbf{w}$ values to produce well-distributed fronts that interpolate various objectives, with the sampling space for $w$ in each dimension set to $\{0.2, 0.4, 0.6, 0.8, 1.0\}$. For $\beta$, we conduct a hyperparameter analysis in \textcolor{blue}{Appendix~\ref{appendix:hyperparameter_beta}} and set $\beta=0.5$.
More implementation details are in \textcolor{blue}{Appendix~\ref{appendix:hyperparameter}}. Computing cost analysis is provided in \textcolor{blue}{Appendix~\ref{appendix:results of effciency}}.

\begin{figure}[!htp]
    \centering
    \includegraphics[width=0.99\linewidth]{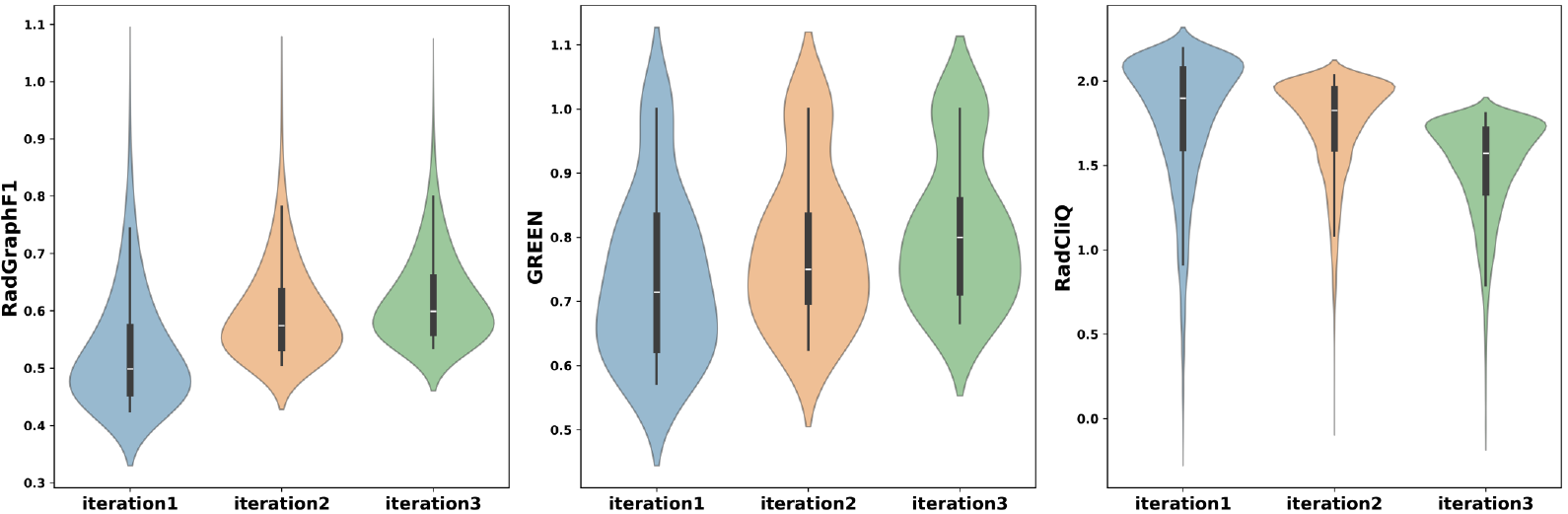}
    \caption{The distribution of preference dataset on different evaluation metrics in three iterations.}\label{fig_pdc}
\end{figure}

\subsection{Results and Analysis}
\textbf{Preference data.}
To verify our pipeline can improve the quality of preference data with each iteration, we provide the violin plots to compare the distribution of the preference data generated in three iterations across three evaluation metrics, RadGraphF1 ($\uparrow$), GREEN($\uparrow$), and RadCliQ($\downarrow$). Since $y^l$ is randomly selected in the PDC module, Figure \ref{fig_pdc} only shows the results of the chosen responses $y^w$. Figure \ref{fig_pdc} shows that with each round of iteration, all quantiles of RadGraphF1 and GREEN gradually increase, while the quantiles of RadCliQ gradually decrease. This demonstrates that as the number of iterations increases, the quality of the preference data improves across all evaluation metrics. 

\textbf{Alignment with different preferences.}
To verify the effectiveness of our proposed method, we test the models obtained at each iteration under four specific preference weight configurations, where $w_1$, $w_2$, and $w_3$ representing the weights of objective represented by RadCliQ($\downarrow$), RadGraphF1, and GREEN, respectively. In each iteration, the first three rows represent the results of fully preferring one of the objectives, while the last row assigns equal weights to all objectives. 

Table~\ref{tab_pref_cxr} and Table~\ref{tab_pref_iu} show the results on the MIMIC-CXR and Iu-Xray dataset. For MIMIC-CXR dataset, we have the following observations:
1) The model outperforms the baseline model across all weight configurations in each iteration.
2) In each iteration, when an objective is fully preferred, the corresponding metric achieves the best performance. When the weights are equally distributed among objectives, the model attains the second-best results across all metrics, indicating that our method can align multiple objectives through weight conditions.
3) The performance trends of BLUE and BERTScore align with those of RadCliQ, as RadCliQ is a linear combination of these metrics. Similarly, the trend of ChexbertF1 correlates with RadGraphF1 due to their similar calculation methods.
4) For the same weight configuration, all metrics show a gradual increase with each iteration. For instance, from iteration 1 to iteration 3, the best results for RadGraphF1, GREEN, and RadCliQ improved from 0.247, 0.326, and 2.62 to 0.273, 0.341, and 2.54, respectively on MIMIC-CXR dataset and from 0.285, 0.482, and 2.57 to 0.308, 0.527, and 2.51, respectively on IU-Xray dataset with the best performance achieved in the third iteration. This demonstrates that our method can continuously enhance the overall performance through multiple iterations. 
5) Similar observations can be drawn on the IU-Xray dataset.

\begin{figure*}[ht]
  \centering
  \begin{subfigure}[c]{0.32\textwidth}
    \centering
    \includegraphics[width=\textwidth]{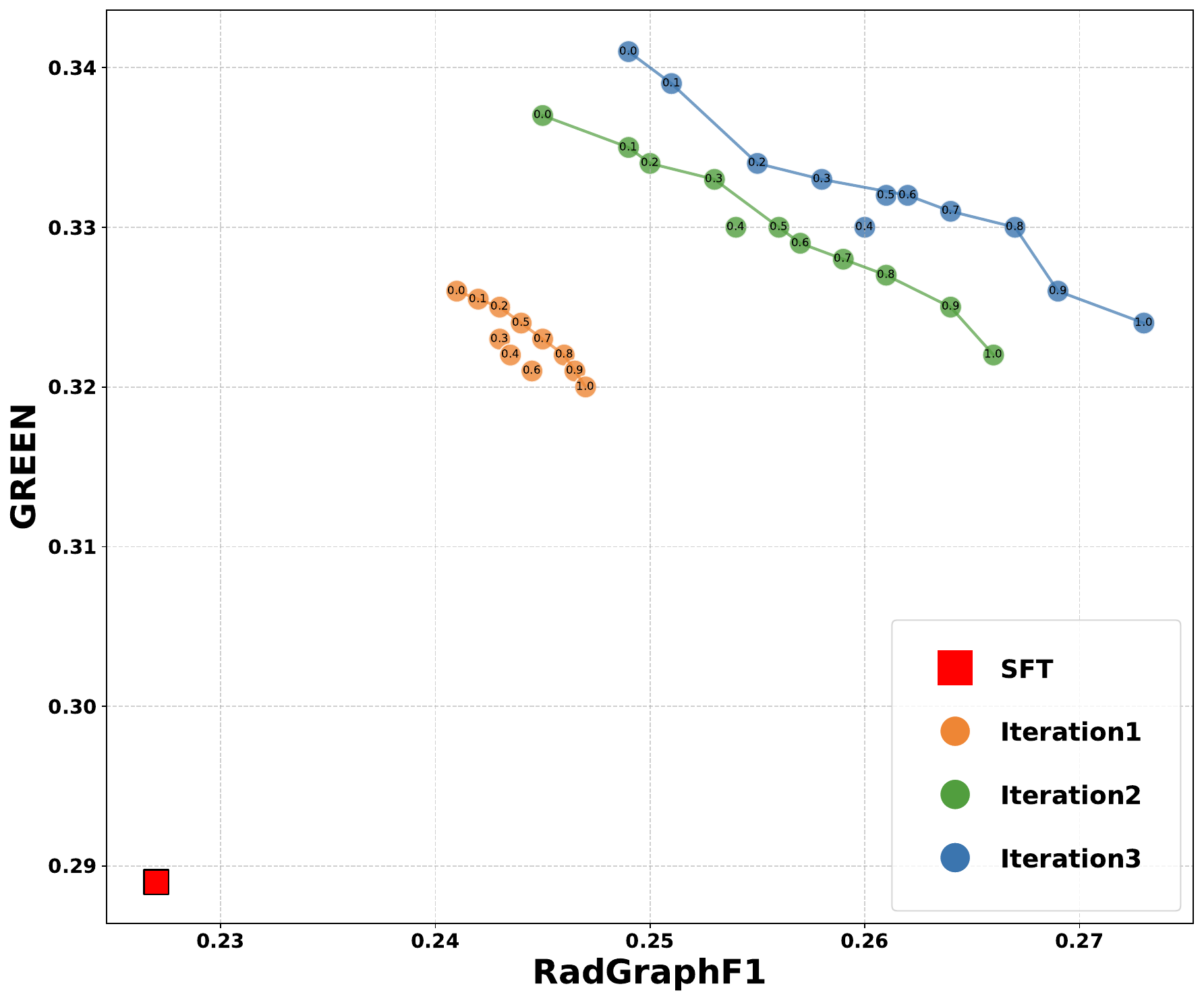}
    \caption{RadGraphF1 and GREEN}
    \label{fig:first_image}
  \end{subfigure}
  \hfill
  \begin{subfigure}[c]{0.32\textwidth}
    \centering
    \includegraphics[width=\textwidth]{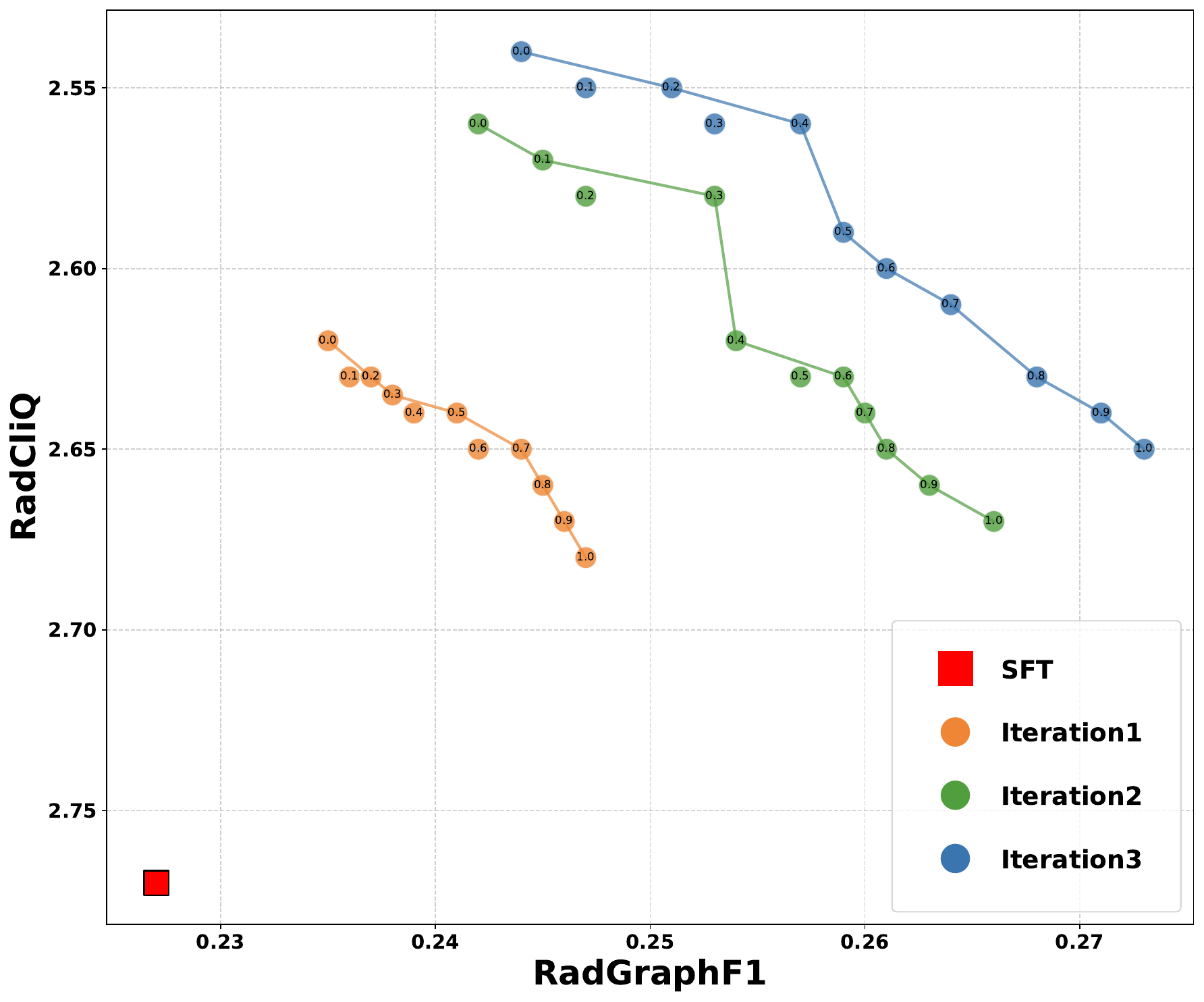}
    \caption{RadGraphF1 and  RadCliQ (reverse)}
    \label{fig:second_image}
  \end{subfigure}
  \hfill
  \begin{subfigure}[c]{0.32\textwidth}
    \centering
    \includegraphics[width=\textwidth]{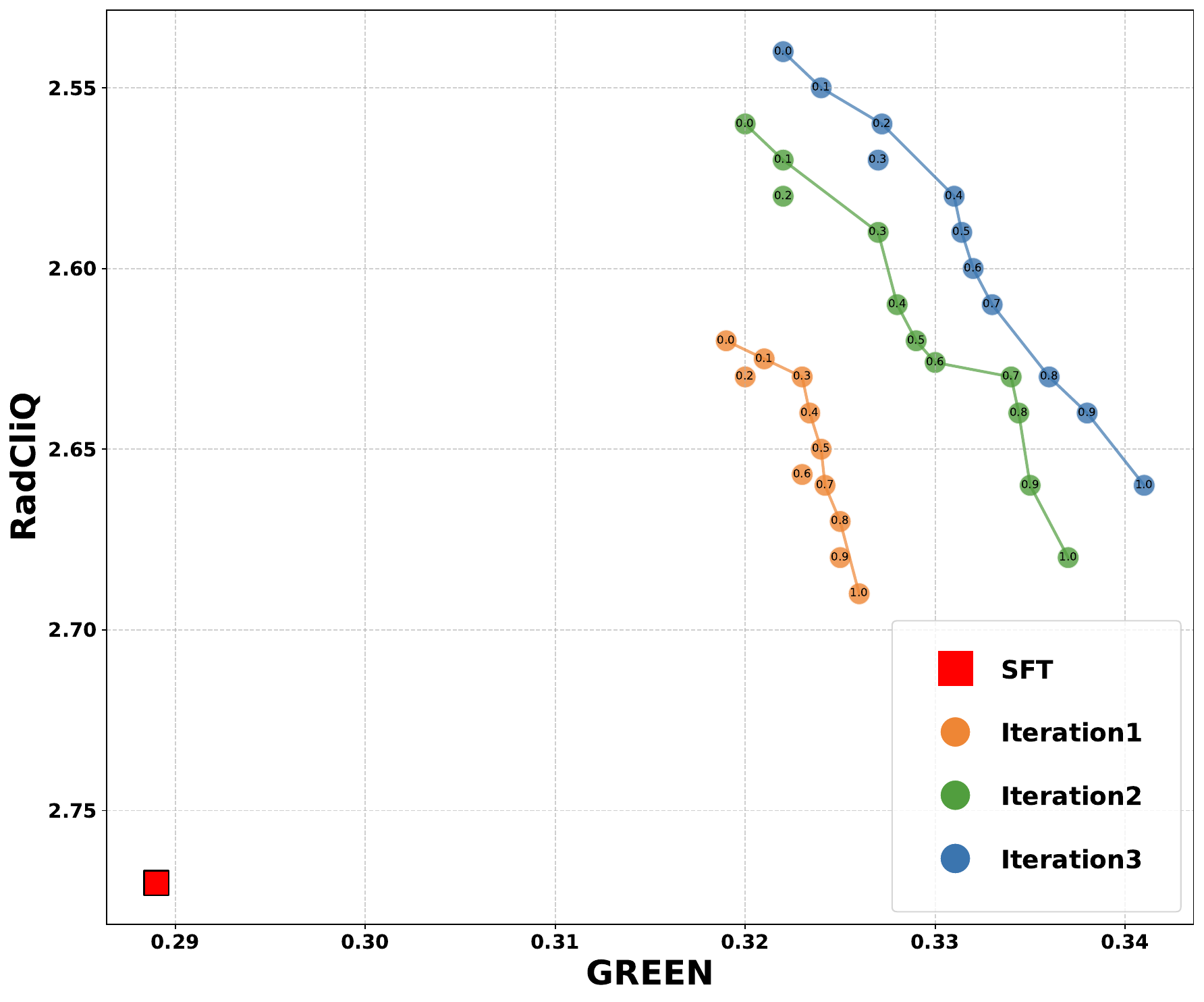}
    \caption{GREEN and RadCliQ (reverse)}
    \label{fig:second_image}
  \end{subfigure}
  \caption{The multi-objective alignment fronts across three iterations.}
  \label{fig_fronts}
\end{figure*}

\begin{table*}[!ht]
    \renewcommand{\arraystretch}{1.0}
    \centering
    \scalebox{0.9}{
    \begin{tabular}{c|ccc|cccc|cc|c}
        \hline
        Model & $w_1$ & $w_2$ & $w_3$ & B1 & B4 & BERTScore & RadCliQ & RadGraphF1 & ChexbertF1 & GREEN \\ 
        \hline
        SFT & - & - & - & 0.398 & 0.112 & 0.857 & 2.77 & 0.227 & 0.476 & 0.289 \\ 
        \hline
        
        \multirow{4}{*}{Iteration 1} 
        & 1 & 0 & 0 & \textcolor{blue}{\textbf{0.418}} & \textcolor{blue}{\textbf{0.119}} & \textcolor{blue}{\textbf{0.869}} & \textcolor{blue}{\textbf{2.62}} & 0.238 & 0.481 & 0.319 \\ 
        & 0 & 1 & 0 & 0.414 & 0.115 & 0.860 & 2.68 & \textcolor{blue}{\textbf{0.247}} & \textcolor{blue}{\textbf{0.499}} & 0.320 \\ 
        & 0 & 0 & 1 & 0.412 & 0.115 & 0.861 & 2.69 & 0.241 & 0.484 & \textcolor{blue}{\textbf{0.326}} \\ 
        & 1/3 & 1/3 & 1/3 & \underline{0.415} & \underline{0.117} & \underline{0.865} & \underline{2.65} & \underline{0.244} & \underline{0.491} & \underline{0.323} \\ 
        \hline
        
        \multirow{4}{*}{Iteration 2} 
        & 1 & 0 & 0 & \textcolor{blue}{\textbf{0.426}} & \textcolor{blue}{\textbf{0.125}} & \textcolor{blue}{\textbf{0.879}} & \textcolor{blue}{\textbf{2.56}} & 0.242 & 0.483 & 0.320 \\ 
        & 0 & 1 & 0 & 0.417 & 0.116 & 0.864 & 2.67 & \textcolor{blue}{\textbf{0.266}} & \textcolor{blue}{\textbf{0.505}} & 0.322 \\ 
        & 0 & 0 & 1 & 0.415 & 0.115 & 0.866 & 2.68 & 0.245 & 0.486 & \textcolor{blue}{\textbf{0.337}} \\ 
        & 1/3 & 1/3 & 1/3 & \underline{0.419} & \underline{0.119} & \underline{0.874} & \underline{2.63} & \underline{0.251} & \underline{0.494} & \underline{0.325} \\ 
        \hline
        
        \multirow{4}{*}{Iteration 3} 
        & 1 & 0 & 0 & \textcolor{blue}{\textbf{0.428}}* & \textcolor{blue}{\textbf{0.129}}* & \textcolor{blue}{\textbf{0.885}}* & \textcolor{blue}{\textbf{2.54}}* & 0.244 & 0.486 & 0.322 \\ 
        & 0 & 1 & 0 & 0.418 & 0.117 & 0.872 & 2.65 & \textcolor{blue}{\textbf{0.273}}* & \textcolor{blue}{\textbf{0.516}}* & 0.324 \\ 
        & 0 & 0 & 1 & 0.417 & 0.116 & 0.873 & 2.66 & 0.249 & 0.484 & \textcolor{blue}{\textbf{0.341}}* \\ 
        & 1/3 & 1/3 & 1/3 & \underline{0.421} & \underline{0.121} & \underline{0.879} & \underline{2.61} & \underline{0.254} & \underline{0.504} & \underline{0.327} \\ 
        \hline
    \end{tabular}
    }
    \caption{Results on MIMIC-CXR dataset under four preference weights across three iterations, where $w_1$, $w_2$, and $w_3$ correspond to the weights of the objective of RadCliQ, RadGraphF1, and GREEN, respectively. SFT represents the results produced by the baseline model. Bold \textcolor{blue}{\textbf{blue}} denote the best results in each iteration, while \underline{underlined} indicates the second-best results. An asterisk (*) signifies the best result among all iterations.}
    \label{tab_pref_cxr}
\end{table*}

\begin{table*}[!ht]
    \renewcommand{\arraystretch}{1.0}
    \centering
    \scalebox{0.9}{
    \begin{tabular}{c|ccc|cccc|cc|c}
        \hline
        Model & $w_1$ & $w_2$ & $w_3$ & B1 & B4 & BERTScore & RadCliQ & RadGraphF1 & ChexbertF1 & GREEN \\ 
        \hline
        SFT & - & - & - & 0.401 & 0.098 & 0.871 & 2.64 & 0.274 & 0.211 & 0.457 \\ 
        \hline
        
        \multirow{4}{*}{Iteration 1} 
        & 1 & 0 & 0 & \textcolor{blue}{\textbf{0.411}} & \textcolor{blue}{\textbf{0.107}} & \textcolor{blue}{\textbf{0.879}} & \textcolor{blue}{\textbf{2.57}} & 0.277 & 0.211 & 0.461 \\ 
        & 0 & 1 & 0 & 0.402 & 0.101 & 0.871 & 2.61 & \textcolor{blue}{\textbf{0.285}} & \textcolor{blue}{\textbf{0.218}} & 0.459 \\ 
        & 0 & 0 & 1 & 0.404 & 0.102 & 0.872 & 2.64 & 0.275 & 0.212 & \textcolor{blue}{\textbf{0.482}} \\ 
        & 1/3 & 1/3 & 1/3 & \underline{0.407} & \underline{0.102} & \underline{0.874} & \underline{2.60} & \underline{0.281} & \underline{0.215} & \underline{0.467} \\ 
        \hline
        
        \multirow{4}{*}{Iteration 2} 
        & 1 & 0 & 0 & \textcolor{blue}{\textbf{0.426}} & \textcolor{blue}{\textbf{0.124}} & \textcolor{blue}{\textbf{0.885}} & \textcolor{blue}{\textbf{2.55}} & 0.279 & 0.217 & 0.465 \\ 
        & 0 & 1 & 0 & 0.405 & 0.103 & 0.874 & 2.60 & \textcolor{blue}{\textbf{0.299}} & \textcolor{blue}{\textbf{0.227}} & 0.468 \\ 
        & 0 & 0 & 1 & 0.404 & 0.102 & 0.875 & 2.61 & 0.276 & 0.214 & \textcolor{blue}{\textbf{0.513}} \\ 
        & 1/3 & 1/3 & 1/3 & \underline{0.415} & \underline{0.114} & \underline{0.879} & \underline{2.58} & \underline{0.291} & \underline{0.222} & \underline{0.484} \\ 
        \hline
        
        \multirow{4}{*}{Iteration 3} 
        & 1 & 0 & 0 & \textcolor{blue}{\textbf{0.431}}* & \textcolor{blue}{\textbf{0.131}}* & \textcolor{blue}{\textbf{0.889}}* & \textcolor{blue}{\textbf{2.51}}* & 0.282 & 0.219 & 0.481 \\ 
        & 0 & 1 & 0 & 0.411 & 0.105 & 0.877 & 2.59 & \textcolor{blue}{\textbf{0.308}}* & \textcolor{blue}{\textbf{0.232}}* & 0.479 \\ 
        & 0 & 0 & 1 & 0.409 & 0.107 & 0.880 & 2.61 & 0.280 & 0.216 & \textcolor{blue}{\textbf{0.527}}* \\ 
        & 1/3 & 1/3 & 1/3 & \underline{0.421} & \underline{0.122} & \underline{0.882} & \underline{2.56} & \underline{0.305} & \underline{0.225} & \underline{0.512} \\ 
        \hline
    \end{tabular}}
    \caption{Results on IU-Xray dataset under four preference weights across three iterations, where $w_1$, $w_2$, and $w_3$ correspond to the weights of the objective of RadCliQ, RadGraphF1, and GREEN, respectively. SFT represents the results produced by the baseline model. Bold \textcolor{blue}{\textbf{blue}} denote the best results in each iteration, while \underline{underlined} indicates the second-best results. An asterisk (*) signifies the best result among all iterations.}
    \label{tab_pref_iu}
\end{table*}

\textbf{Multi-objective alignment fronts.}
To verify the performance of multi-objective alignment among iterations, we evaluated the model with a set of sampled weights to show the Pareto fronts of the learned multi-objective policy, the results are shown in Figure~\ref{fig_fronts}, and the three paired objectives are shown in the subfigures, i.e., 
(a) RadGraphF1 vs. GREEN, (b) RadGraphF1 vs. RadCliQ (reverse), and (c) GREEN vs. RadCliQ (reverse). In each figure, the distribution of each front reveals how the model behaves under varying weight priorities, where each point's weight on the front corresponds to the weight assigned to the horizontal axis metric, ranging from 0-1. 

According to the visualization, we find that starting from the initial SFT model, the Pareto fronts of all objective pairs notably towards the upper-right quadrant in iterations. This indicates the proposed iterative data generation and alignment process ensures continuous policy improvement overall objectives, which has also been verified in our theoretical analysis. Meanwhile, in each front, the performance of both axes changes smoothly as we move along the front. 
Importantly, when the horizontal metric generally improves with its weight, the vertical metric maintains the best possible value according to the Pareto principle: there are no solutions that simultaneously improve both metrics. In the last iterations, our method exhibits extensive and continuous Pareto fronts, which demonstrate the model's enhanced ability with improved data quality after self-iterations. The final results give well-balanced Pareto-efficient solutions that cater to a wide range of objectives, representing the optimal trade-off frontier for each metric pair.

\begin{table*}[!htp]
\begin{center}\renewcommand{\arraystretch}{1.0} 
\centering\scalebox{0.9}{
 \begin{tabular}{c|c|ccccccc}\hline
        Model & Size & B1 & B4 & BERTScore & RadCliQ & RadGraphF1 & ChexbertF1 & GREEN \\ \hline
        R2Gen & 78.5M & 0.353 & 0.103 & 0.866 & 2.89 & 0.195 & 0.276 & 0.306 \\ 
        CMN & 59.1M & 0.353 & 0.108 & 0.867 & 2.87 & 0.199 & 0.278 & 0.308 \\ 
        CMN+RL & 60.8M & 0.381 & 0.109 & 0.871 & 2.83 & 0.214 & 0.292 & 0.315 \\ 
        PromptMRG & 219.9M & 0.398 & 0.112 & 0.857 & 2.77 & 0.227 & 0.476 & 0.289 \\
        MPO & 63.3M & 0.416 & \textbf{0.139} & 0.878 & 2.63 & 0.257 & 0.353 & 0.324 \\ \hline
    Ours (iteration 1) &\multirow{3}{*}{230.1M} & 0.418 & 0.119 & 0.869 & 2.62 & 0.247 & 0.499 & 0.326 \\ 
    Ours (iteration 2) &~ & \underline{0.426} & 0.125 & \underline{0.879} & 2.56 & 0.266 & \underline{0.505} & 0.337 \\ 
    Ours (iteration 3)&~ & \textbf{0.428} & \underline{0.129} & \textbf{0.885} & \underline{2.54} & \underline{0.273} & \textbf{0.516} & \underline{0.341} \\ \hline
    MedVersa    &7B	&0.28	&0.09	&0.711	&\textbf{2.45}	&\textbf{0.289}	&0.471	&\textbf{0.381} \\
    MiniGPT-Med &7B & 0.191 & 0.012 & 0.636 & 2.95 & 0.164 & 0.172 & 0.211 \\ 
        CheXagent &8B & 0.172 & 0.021 & 0.669 & 2.88 & 0.19 & 0.265 & 0.268 \\ \hline
    \end{tabular}}
\end{center}
\vspace{-1em}
\caption{Comparison with existing RRG methods on MIMIC-CXR dataset. The best results are highlighted in bold and \underline{underlined} indicates the second-best results.}\label{tab_previous}
\end{table*}

\begin{table*}[!htp]
\begin{center}\renewcommand{\arraystretch}{1.0} 
\centering\scalebox{0.9}{
 \begin{tabular}{c|c|ccccccc}\hline
        Model & Size & B1 & B4 & BERTScore & RadCliQ & RadGraphF1 & ChexbertF1 & GREEN \\ \hline
        R2Gen & 78.5M & 0.363 & 0.073 & 0.861 & 2.79 & 0.187 & 0.154 & 0.482 \\ 
        CMN & 59.1M & 0.39 & 0.085 & 0.862 & 2.75 & 0.186 & 0.155 & 0.516 \\  
        PromptMRG & 219.9M & 0.401 & 0.098 & 0.871 & 2.60 & 0.274 & 0.211 & 0.457 \\ \hline
    Ours (iteration 1) &\multirow{3}{*}{230.1M} & 0.411 & 0.107 & 0.879 & 2.57 & 0.285 & 0.218 & 0.482 \\ 
    Ours (iteration 2) &~ & \underline{0.426} & \underline{0.124} & \underline{0.885} & \underline{2.55} & \underline{0.299} & \underline{0.227} & 0.513 \\ 
    Ours (iteration 3)&~ & \textbf{0.431} & \textbf{0.131} & \textbf{0.889} & \textbf{2.51} & \textbf{0.308} & \textbf{0.232} & \textbf{0.527} \\ \hline
    MedVersa    &7B	&0.247	&0.047	&0.884	& 2.71	& 0.209	&0.217	&\underline{0.516} \\
    CheXagent &8B & 0.191 & 0.036 & 0.876 & 2.81 & 0.184 & 0.097 & 0.407 \\ \hline
    \end{tabular}}
\end{center}
\caption{Comparison with existing RRG methods on IU-Xray dataset. The best results are highlighted in \textbf{bold} and \underline{underlined} indicates the second-best results.}\label{tab_previous_iu}
\end{table*}

\textbf{Comparison with other RRG works.}
The comparison methods include traditional approaches such as R2Gen \cite{R2Gen}, CMN \cite{R2GenCMN}, CMN+RL \cite{qin2022reinforced}, PromptMRG \cite{jin2024promptmrg}, and MPO \cite{MPO}, as well as VLM-based methods like CheXagent \cite{CheXagent}, MedVersa \cite{MedVersa}, and MiniGPT-Med \cite{Minigpt-med}. The results on the MIMIC-CXR dataset are presented in Table~\ref{tab_previous}. Compared to radiology metrics, all VLM-based models perform significantly worse than traditional RRG models on NLG metrics, likely due to specialist foundation models prioritizing clinical efficiency in report generation. Our method outperforms all traditional methods in the second iteration and achieves the best performance in the third iteration.
When compared to VLM-based models, our method shows significant superiority over CheXagent and MiniGPT-Med, and its performance is comparable to that of MedVersa. This indicates that our proposed method can generate reports aligned with multi-objective preferences while also improving overall performance with a lightweight model.

We also compared our methods with existing RRG models on the IU-Xray dataset. Following the experimental setup outlined in PromptMRG, all the data in the IU-Xray dataset are used to test. 
Table~\ref{tab_previous_iu} shows the results compared with traditional approaches such as PromptMRG \cite{jin2024promptmrg} , as well as VLM-based models like CheXagent \cite{CheXagent} and MedVersa \cite{MedVersa}. Our method achieves the best performance on all metrics. Like MIMIC-CXR, all VLM-based models significantly underperform traditional RRG models in NLG metrics.


\section{Related Work}\label{section_2}
Existing RRG methods can be divided into supervised fine-tuning (SFT)-based and post-training-based methods from the view of model learning.

\textbf{SFT-based methods}. 
These methods typically improve the quality of RRG by refining different network architectures or injecting external knowledge. The former usually adopts CNN-Transformer as the basic architecture and introduces new modules, such as cross-modal memory network \cite{R2GenCMN, MAN} and memory enhancement module \cite{cao2023mmtn} to enhance cross-modal pattern learning. Another line of research focuses on incorporating external knowledge, such as knowledge graphs \cite{kale2023knowledge, KIUT}, disease tags \cite{jin2024promptmrg}, and retrieved reports \cite{liu2024bootstrapping}, into the RRG pipeline to enhance the quality of generated reports. Recently, several medical visual language models (VLMs) derived from large language models and tailored for medical applications have been proposed, such as MiniGPT-Med\cite{Minigpt-med}, MedVersa \cite{MedVersa}, and CheXagent \cite{CheXagent}, which can perform multiple medical tasks.

The aforementioned SFT-based RRG methods have made significant progress. However, due to the limited scale of high-quality annotated data, these methods have the potential risk of overfitting and cannot be generalized outside of the dataset.

\textbf{Post-training-based methods.} Post-training techniques have been applied to further improve the generation capabilities of existing RRG models through RL with reward models \cite{wang2021self, wang2022medical, qin2022reinforced, MPO} or direct preference alignment without reward models \cite{hein2024preference}. For example, \cite{wang2021self, qin2022reinforced} use off-the-shelf natural language generation (NLG) metrics such as BLEU as the reward function and maximize the average reward expectation of over the training data to improve the quality of report generation. \citet{delbrouck2022improving} apply RL to enhance report generation's factual completeness and correctness by designing semantic relevance metrics based on entity coverage and relations. MPO \cite{MPO} uses a multi-dimensional preference vector as a condition for the RRG model, and optimizes the weighted multi-dimensional reward through RL to align with human preferences. However, such an RL alignment process can still be limited to the data coverage of \citep{xiong2024iterative}. Further, \citet{hein2024preference} construct a preference dataset with wide coverage using the 8B foundation model CheXagent \cite{CheXagent} and investigates five offline direct preference alignment algorithms based on this preference dataset. Although such a method achieves strong performance, it requires a large foundation and score models, which are much more costly than other methods. 

Unlike previous post-training methods \cite{hein2024preference, MPO} in RRG, our method greatly extends the data coverage of the offline dataset via an iterative process and improves the performance via multi-objective optimization.

\section{Conclusion}

This paper proposes an OISA method for RRG, addressing the limitations of existing models that rely on fixed datasets. By employing self-generation, self-evaluation, self-alignment, and self-iteration, we have established a post-training framework that expands data quality and performs alignment gradually with multi-objectives. Theoretical results show our method leads to tight regret bound under linear reward assumptions. The experimental results highlight the potential of lightweight RRG models to deliver high-quality reports for multiple objectives and achieve Pareto-optimal alignment performance among iterations, which underscore the value of self-generated data and iterative improvements in adapting to diverse clinical needs.

\section{Limitations}
Due to computing constraints, we only focus on a single traditional RRG model, other models with different sizes and architectures would be promising to explore in the future. Further, we adopt the existing evaluation metrics of RRG model to represent user preferences, which may not be consistent with the actual needs of clinicians. It would be important to define fine-grained and informative metrics to construct the preference dataset and perform multi-objective alignment in future works. 

\section{Acknowledgment}
This work is supported by the National Natural Science Foundation of China under grants No. 62306115 and No. 62476087.
\bibliography{custom}

\clearpage
\appendix
\onecolumn

\section{Deduplication Rules in Self-Generation}\label{appendix:pdc}
The deduplication rules are as follows:

(\romannumeral1) \textbf{Patient level.} For reports from the same patient but different views, we calculate the NLG metric (BERTScore \cite{BERTScore}) and only retain the report with the higher BERTScore value. From the original training set of 227,835 reports, we retain 130,534 reports.

(\romannumeral2) \textbf{Disease label level.} We then group the reports obtained in (\romannumeral1) according to the 14 disease labels extracted by CheXbert~\cite{CheXbert}, resulting in 579 groups, with the largest group containing 16,549 reports and the smallest group containing 1 report. Within each group, where groups with more than 2 reports are further deduplicated as follows:

1) Discard reports with a BERTScore value below 0.5.

2) Compute the BERTScore for every pair of reports in the group. If the BERTScore value exceeds the threshold of 0.8, indicating that the two reports are very similar, we only keep the report with a higher BERTScore value compared to the ground truth report.
The choice of the threshold is empirical. A lower threshold may result in removing too many dissimilar reports, while a higher threshold may result in discarding more relatively similar reports, leading to a decrease in diversity.

After this second deduplication, the total number of reports is reduced to 98,753, with the largest group containing 11,324 reports, while the smallest group remains 1 report.

\section{How to calculate $K_c$ in Self-Evaluation}\label{appendix:KC}

For each preference dimension $k$, we obtain a candidate report set $\mathbf{Y} =\{\mathbf{Y}^c\}_{c=1}^{C}$ via the Self-Generation process, where $C$ is the total number of groups, $\mathbf{Y}^c=\{(x_j, y_j)\}_{j=1}^{N_c}$, and $N_c$ is the total number of samples in $c$-th group. Within each group, we rank the reports based on the metric \( M_k \). The total number of pairs in each preference dataset is $K$. We denote the number of reports to be sampled as $K_r$ and the number of groups with more than one report as $Q$;
We then perform sampling within each group according to the following rules.
\begin{itemize}
        \item Calculate the average number of samples to be sampled for each group as $\mu = \left\lceil \frac{K_r}{Q} \right\rceil $.
        \item For groups with fewer than $ \mu $ reports, we sample all reports within those groups, i.e., $K_c=N_c$, resulting in a total of \( K_1 \) reports.
        \item  For groups with more than \( \mu \) reports, we first sample $K_c = \mu $ from each group, resulting in a total of $K_2$ reports.
        \item Update the number of reports that remain to be selected as: $ K_r \leftarrow (K_r- K_1 -K_2)$, the number of groups with more than one report $Q$. Repeat the above process until \( K_r = 0 \).
\end{itemize}

\section{Proof of Theorem~\ref{thm:subop_gap}}\label{appendix:proof_of_subop_gap}
To bound the sub-optimality gap of the policy derived from multi-objective direct preference optimization, we first introduce an important lemma on the sub-optimality gap bound of the policy derived from single-objective direct preference optimization.
\begin{lemma}\label{lemma:single_obj_subopt}
Given a dataset $\mathcal{D}_k$ for a specific preference dimension $k$, the MLE estimated reward model is denoted as $\hat{\theta}_k$. Denote the optimal policy w.r.t. RL tuning objective $J(\pi; \hat{\theta}_k)$ with the estimated reward model $\hat{\theta}_k$ as $\hat{\pi} = \mathop{\arg\max}_{\pi} J(\pi; \hat{\theta}_k)$, and the optimal policy w.r.t. RL tuning objective with the ground-truth reward model $\theta^\star_k$ as $\pi^\star = \mathop{\arg\max}_{\pi} J(\pi; \theta^\star_k)$. For any $\lambda > 0$, $\beta > 0$, with probability at least $1 - \delta$, $\hat{\pi}$ satisfies
\begin{align}
    \text{\rm SubOpt}(\hat{\pi}) \leq 2\varrho \cdot \| \mathbb{E}_{x \sim \rho} [ \phi(x, \pi^\star(x)) ] \|_{(\Sigma_{\mathcal{D}_k} + \lambda I)^{-1}}. \label{equ:single_dpo_subopt}
\end{align}
\end{lemma}
\begin{proof}[Proof of Lemma~\ref{lemma:single_obj_subopt}]
We first decompose the sub-optimality gap defined in Eqn. (\ref{equ:single_dpo_subopt}).
\begin{align}
    \text{\rm SubOpt}(\hat{\pi}) &= J(\pi^\star, \theta^\star_k) - J(\hat{\pi}, \theta^\star_k) \nonumber\\
    & = \mathbb{E}_{x \sim \rho, y \sim \pi^\star(y | x)} \Big[ r^\star_k(x,y) - \beta\log \frac{\pi^\star(y | x)}{\pi_{\text{ref}}(y | x)} \Big] - \mathbb{E}_{x \sim \rho, y \sim \hat{\pi}(y | x)} \Big[ r^\star_k(x,y) - \beta\log \frac{\hat{\pi}(y | x)}{\pi_{\text{ref}}(y | x)} \Big]  \nonumber\\
    & = \underbrace{\mathbb{E}_{x \sim \rho, y \sim \pi^\star(y | x)} [ r^\star_k(x,y) - \hat{r}_k (x, y)]}_{\text{Term (i)}} + \underbrace{\mathbb{E}_{x \sim \rho, y \sim \hat{\pi}(y | x)} [ \hat{r}_k (x, y) - r^\star_k (x, y)]}_{\text{Term (ii)}} + \underbrace{J(\pi^\star, \hat{\theta}_k) - J(\hat{\pi},\hat{\theta}_k)}_{\text{Term (iii)}} \label{equ:single_obj_decompose}
\end{align}

\textbf{Term (i).} Under Assumption~\ref{assump:linear_r}, we can rewrite Term (i) in Eqn. (\ref{equ:single_obj_decompose}) as
\begin{align}
    \text{Term (i)} &= \mathbb{E}_{x \sim \rho, y \sim \pi^\star(y | x)} [ (\theta^\star_k - \hat{\theta}_k)^\top \phi(x, y)] \nonumber \\
    & \leq \| \theta^\star_k - \hat{\theta}_k \|_{\Sigma_{\mathcal{D}_k} + \lambda I} \cdot \| \mathbb{E}_{x \sim \rho, y \sim \pi^\star(y | x)} [ \phi(x, y) ] \|_{(\Sigma_{\mathcal{D}_k} + \lambda I)^{-1}} \nonumber \\
    & \leq \varrho \cdot \| \mathbb{E}_{x \sim \rho, y \sim \pi^\star(y | x)} [ \phi(x, y) ] \|_{(\Sigma_{\mathcal{D}_k} + \lambda I)^{-1}}, \label{equ:term_i_bound}
\end{align}
where the first inequality results from Cauchy-Schwarz inequality, and the last inequality is obtained by using Lemma~\ref{lemma:bound_MLE}.

\textbf{Term (ii).} Akin to the derivation of Eqn. (\ref{equ:term_i_bound}), we also have
\begin{align}
    \text{Term (ii)} &= \mathbb{E}_{x \sim \rho, y \sim \hat{\pi}(y | x)} [ (\hat{\theta}_k - \theta^\star_k)^\top \phi(x, y)] \nonumber \\
    & \leq \varrho \cdot \| \mathbb{E}_{x \sim \rho, y \sim \hat{\pi}(y | x)} [ \phi(x, y) ] \|_{(\Sigma_{\mathcal{D}_k} + \lambda I)^{-1}}. \label{equ:term_ii_bound}
\end{align}

\textbf{Term (iii).} Since $\hat{\pi}$ satisfies $\hat{\pi} = \mathop{\arg\max}_\pi J(\pi, \hat{\theta}_k)$, we have
\begin{align}
    J(\pi^\star, \hat{\theta}_k) \leq J(\hat{\pi}, \hat{\theta}_k). \label{equ:term_iii_bound}
\end{align}

\textbf{Combining three terms together.} Substituting Eqs. (\ref{equ:term_i_bound}--\ref{equ:term_iii_bound}) into Eqn. (\ref{equ:single_obj_decompose}), we get
\begin{align}
    \text{\rm SubOpt}(\hat{\pi}) &\leq \varrho \cdot \big( \| \mathbb{E}_{x \sim \rho} [ \phi(x, \pi^\star(x)) ] \|_{(\Sigma_{\mathcal{D}_k} + \lambda I)^{-1}} + \| \mathbb{E}_{x \sim \rho} [ \phi(x, \hat{\pi}(x)) ] \|_{(\Sigma_{\mathcal{D}_k} + \lambda I)^{-1}} \big). \label{equ:subopt_with_two_semi_norm}
\end{align}
Here the term $\| \mathbb{E}_{x \sim \rho} [ \phi(x, \pi(x)) ] \|_{(\Sigma_{\mathcal{D}_k} + \lambda I)^{-1}}$ for any $\pi$ is equivalent to a measurement of how well the current dataset $\mathcal{D}_k$ covers the distribution of responses generated by the given policy $\pi$. Recalling that the preference dataset $\mathcal{D}_{k}^{(i)}$ in every iteration $i > 1$ is collected by the initial reference policy $\pi_{\text{ref}}$ (i.e., the resulted policy $\hat{\pi}^{(i - 1)}$ in the last iteration $i - 1$) in the iterative RLHF paradigm, the target vector $\mathbb{E}_{x \sim \rho}[ \phi(x, \hat{\pi}^{(i)}(x) ]$ induced by the optimized policy $\hat{\pi}^{(i)}$ in this iteration $i$ would overlaps more with the dataset $\mathcal{D}_{k}^{(i)}$ than that of the optimal policy $\pi^\star$. Therefore, the following inequality generally holds
\begin{align}
    \| \mathbb{E}_{x \sim \rho} [ \phi(x, \pi^\star(x)) ] \|_{(\Sigma_{\mathcal{D}_k} + \lambda I)^{-1}} \geq \| \mathbb{E}_{x \sim \rho} [ \phi(x, \hat{\pi}(x)) ] \|_{(\Sigma_{\mathcal{D}_k} + \lambda I)^{-1}} \label{equ:measurement_diff}
\end{align}
By combining Eqn. (\ref{equ:measurement_diff}) and Eqn. (\ref{equ:subopt_with_two_semi_norm}), we can conclude the proof of Eqn. (\ref{equ:single_dpo_subopt}).
\end{proof}
Now we are ready to prove our main theorem.
\begin{proof}[Proof of Theorem~\ref{thm:subop_gap}]
Given any reward $\mathbf{r} = [r_1, \ldots, r_N]^\top = [\theta_1, \ldots, \theta_N]^T \phi (x,y)$, denoting $\boldsymbol{\theta} = [\theta_1, \ldots, \theta_N] \in \mathbb{R}^{d \times N}$, the objective during RL tuning phase in multi-objective RLHF can be rewritten as
\begin{align*}
    J(\pi_{\mathbf{w}}, \boldsymbol{\theta}) &= \mathbb{E}_{x, y} \left[\mathbf{w}^\top \mathbf{r}(x, y) - \beta \log \frac{\pi_{\mathbf{w}}(y | x)}{\pi_{\text{ref}}(y | x)} \right] \\
    & = \mathbb{E}_{x, y} \left[ \mathbf{w}^\top \mathbf{r}(x, y) - \beta \cdot \mathbf{w}^\top \mathbf{1} \cdot \log \frac{\pi_{\mathbf{w}}(y | x)}{\pi_{\text{ref}}(y | x)} \right] \\
    & = \mathbf{w}^\top \mathbb{E}_{x, y} \left[  \boldsymbol{\theta}^\top \phi(x, y) -  \mathbf{1} \cdot \beta \log \frac{\pi_{\mathbf{w}}(y | x)}{\pi_{\text{ref}}(y | x)} \right].
\end{align*}
The term $\mathbb{E}_{x, y} [  \boldsymbol{\theta}^\top \phi(x, y) -  \mathbf{1} \cdot \beta \cdot \text{KL}(\pi_{\mathbf{w}}||\pi_{\text{ref}}) ]$ is a vector of which every element is equal to RL tuning objective w.r.t. the reward model from the corresponding preference dimension. Further, we can rewrite the sub-optimality gap of $\hat{\pi}_{\mathbf{w}} = \mathop{\arg\max}_{\pi} J(\pi_{\mathbf{w}}, \hat{\boldsymbol{\theta}}_{\text{MLE}})$,
\begin{align}
    \text{\rm SubOpt} (\hat{\pi}_\mathbf{w}) &= J (\pi^\star, \boldsymbol{\theta}^\star) - J (\hat{\pi}_\mathbf{w}, \boldsymbol{\theta}^\star) \nonumber\\
    & = \mathbf{w}^\top \Big(\mathbb{E}_{x, y \sim \pi^\star} \big[ {\boldsymbol{\theta}^{\star}}^\top \phi(x, y) - \mathbf{1} \cdot \beta \log \frac{\pi^\star(y|x)}{\pi_{\text{ref}}(y|x)} \big] \nonumber\\
    &\quad \quad \ - \mathbb{E}_{x, y \sim \hat{\pi}_\mathbf{w}} \big[  {\boldsymbol{\theta}^{\star}}^\top \phi(x, y) -\mathbf{1} \cdot \beta \log \frac{\hat{\pi}_\mathbf{w}(y|x)}{\pi_{\text{ref}}(y|x)} \big] \Big) \nonumber\\
    & = \sum_{k = 1}^{N} w_k \cdot \big(J(\pi^\star, \theta^\star_k) - J(\hat{\pi}_{\mathbf{w}}, \theta^\star_k)\big),\label{equ:multi_obj_decompose}
\end{align}
where $\mathbf{w} = [w_1, \ldots, w_N]^\top$. From Lemma~\ref{lemma:single_obj_subopt}, we get the bound of sub-optimality gap on every single preference dimension. Altogether we have with probability $1 - \delta$
\begin{align}
    \text{\rm SubOpt} (\hat{\pi}_\mathbf{w}) \leq 2\varrho \cdot \sum_{k=1}^N w_k\| \mathbb{E}_{x \sim \rho} [ \phi(x, \pi^\star(x)) ] \|_{(\Sigma_{\mathcal{D}_k} + \lambda I)^{-1}}.
\end{align}
\end{proof}

\section{Experiments}
\subsection{Evaluation metrics}\label{appendix:metrics}
BERTScore \cite{BERTScore} is a similarity score derived from contextualized embeddings. RadGraphF1 \cite{jain2021radgraph} and CheXbertF1 \cite{CheXbert} are clinical scores that incorporate clinically relevant dimensions by focusing on predefined pathological entities. RadCliQ \cite{yu2023evaluating} is a composite metric that linearly combines four existing metrics (i.e., BLEU, BERTScore, CheXbert embedding similarities, and RadGraph \cite{jain2021radgraph}) while learning the combination weights from human-annotated error scores to better align with human evaluations. GREEN \cite{GREEN} is specifically designed to assess model errors by prioritizing significant errors that could impact clinical decision-making.

 \subsection{Hyperparameter analysis of $\beta$}\label{appendix:hyperparameter_beta}
To test hyperparameter $\beta$, we limit our hyperparameter search to the candidate set $[0.1, 0.5, 0.8, 1.0]$ and conduct experiments in the first iteration on the MIMIC-CXR dataset. The experimental results in Table \ref{tab_hyper} show that the value of $\beta$ within a reasonable range does not significantly affect the results. Since we are primarily focused on radiology metrics, we set $\beta=0.5$ for all our experiments.

\begin{table}[!ht]
    \centering
    \scalebox{0.8}{
    \begin{tabular}{c|ccccccc}
        \hline
        $\beta$ & B1 & B4 & BERTScore & ChexBertF1 & GREEN & RadGraphF1 & RadCliQ \\
        \hline
        0.1 & 0.416 & 0.118 & 0.865 & 0.478 & 0.312 & 0.235 & 2.68 \\
        0.5 & 0.418 & \textbf{0.119} & \textbf{0.869} & \textbf{0.481} & \textbf{0.319} & \textbf{0.238} & \textbf{2.62} \\
        0.8 & \textbf{0.419} & 0.119 & 0.868 & 0.477 & 0.311 & 0.235 & 2.70 \\
        1.0 & 0.415 & 0.117 & 0.859 & 0.469 & 0.311 & 0.237 & 2.68 \\
        \hline
    \end{tabular}}
    \caption{Hyperparameter analysis of $\beta$.}
    \label{tab_hyper}
\end{table}

\subsection{Implementation details and hyperparameter list}\label{appendix:hyperparameter}

We set PromptMRG \cite{jin2024promptmrg} as our baseline model and conduct three rounds of iterations, each consisting of 60 epochs, each epoch takes about 15 minutes. Our method is implemented in PyTorch and trained on an NVIDIA 4090 GPU with 24GB of memory, using a batch size of 16 and an initial learning rate of 1e-5. We employ the Adam optimizer during training and apply beam search with a width of 3 for inference. The maximum report lengths for MIMIC-CXR and IU-Xray are 150 and 110, respectively. 
In our approach, the weight vector $\mathbf{w}$ is fused with the image features through a multi-head attention mechanism, where $\mathbf{w}$ is the query and the image features are both the key and the value. The experimental results for the other methods are obtained by running the official code with the parameters they specified in the original paper, the split and preprocess of the test dataset aligns with ours. All hyperparameters are listed in Table~\ref{tab:hyperparameters_list}.

 \begin{table}[h]
\centering\scalebox{0.8}{
\begin{tabular}{rc}\hline
\textbf{Hyperparameters}     & \textbf{Value}  \\ \hline
SFT model & PromptMRG \cite{jin2024promptmrg} \\
Max-Report-Length       & 150/110  vs. MIMIC-CXR/IU-Xray  \\
Sampling space of weight $w_i$   &\{0.2, 0.4, 0.6, 0.8, 1.0\}\\
$m$             & 227,835       \\
Preference pair $K$     & 10,000   \\
Optimizer            & Adam    \\
Learning rate        & 1e-5     \\
Group number $C$        & 579    \\
Epochs number              & 60      \\
Iterations number      & 3  \\
preference dimension $N$ & 3        \\

Batch size              & 16       \\
Num Beams                & 3           \\
$\beta$                  &  0.5 \\ \hline
\end{tabular}}\caption{Hyperparameters list}\label{tab:hyperparameters_list} 
\end{table}

\subsection{Computing cost analysis}\label{appendix:results of effciency}

Table~\ref{appendix:tab_oisa_sft_train} shows the computing cost of our baseline model (PromptMRG) in the SFT stage and iterative preference learning stage. The SFT is performed on the training set of MIMIC-CXR with 227K training samples, while preference data is constructed by ourselves with 10K samples. OISA iterates 3 rounds in total. Preference learning has the same model scale as SFT, while it has much less data (10K vs 227K), which makes it more efficient.

\begin{table*}[!htp]
\begin{center}\renewcommand{\arraystretch}{1.0} 
\small
\begin{tabular}{ccccccc}
\hline
Models               & Dataset \& size               & Batch size & Training time (h/epoch) & GPU(GB) & \#Para (M) & FLOPS(G) \\ \hline
SFT                  & MIMIC-CXR(227K) & 6          & 2.39                    & 8.85    & 219.9      & 181.58   \\
OISA (per iteration) & Preference data 10K           & 6          & 0.14                    & 10.24   & 230.1      & 188.67   \\ \hline
\end{tabular}
\caption{Comparative analysis of training time and resource utilization in SFT and our iterative stage.}\label{appendix:tab_oisa_sft_train}
\end{center}
\end{table*}

\begin{table*}[!htp]
\centering
\begin{tabular}{ccc}
\hline
Model       & Model Size & Inference Time (s/report) \\ \hline
CMN         & 59.1M      & 0.38                      \\
CMN+RL      & 60.8M      & 0.38                      \\
PromptMRG   & 219.9M     & 0.874                     \\ \hline
OISA (Ours) & 230.1M     & 0.905                     \\ \hline
MedVersa    & 7B         & 5.11                      \\
CheXagent   & 8B         & 2.3                       \\ \hline
\end{tabular}
\caption{Comparison of model size and inference time on MIMIC-CXR dataset.}\label{appendix:tab_inference_mimic}
\end{table*}

Inference time is mainly determined by model size. As shown in Table~\ref{appendix:tab_inference_mimic}, OISA shows comparable inference speeds against the SFT model PromptMRG, 0.905s vs. 0.874s, and is significantly faster than MedVersa (5.11s) and CheXagent (2.3s).

\end{document}